\documentclass[english,11pt, letter]{article}
\usepackage{geometry}
\usepackage[active]{srcltx}
\usepackage{amsmath}
\usepackage{amsthm}
\usepackage{amssymb}
\usepackage{esint}

\usepackage{algorithm,algorithmic}
\usepackage[utf8]{inputenc} 
\usepackage[T1]{fontenc}    
\usepackage{hyperref}       
\usepackage{url}            
\usepackage{booktabs}       
\usepackage{amsfonts}       
\usepackage{nicefrac}       
\usepackage{microtype}      
\usepackage{amsmath,amsthm}%
\usepackage{amsfonts}%
\usepackage{amssymb}%
\usepackage{graphicx}
\usepackage{caption}
\usepackage{subcaption}
\usepackage{color}
\usepackage{balance}
\usepackage[table]{xcolor}
\usepackage{colortbl}

\newtheorem{thm}{Theorem}

\newtheorem{claim}[thm]{Claim}

\newtheorem{defn}[thm]{Definition}

\newcommand{\W}{\mathbf{W}}
\renewcommand{\S}{\mathbf{S}}
\newcommand{\s}{\mathbf{s}}

\newcommand{\R}{\mathbf{R}}
\newcommand{\bR}{\mathbb{R}}

\newcommand{\argsort}{\textup{\textrm{argsort}}}

\newtheorem{theorem}{Theorem}

\newtheorem{lemma}[theorem]{Lemma}

\newtheorem{observation}[theorem]{Observation}

\hypersetup{colorlinks=true,citecolor=red}

\begin{document}


\title{Learning to Partition using Score Based Compatibilities}

\author{Arun Rajkumar\\
Conduent Labs - India\\
arun.rajkumar@conduent.com\and Koyel Mukherjee\\
IBM Research - India\\
kmukherj@in.ibm.com\and  Theja Tulabandhula\\
Univ. of Illinois Chicago\\
tt@theja.org}

\maketitle

\begin{abstract}
We study the problem of learning to partition users into groups, where one must learn the compatibilities between the users to achieve optimal groupings. We define four natural objectives that optimize for average and worst case compatibilities  and propose new algorithms for adaptively learning optimal groupings. When we do not impose any structure on the compatibilities, we show that the group formation objectives considered are $NP$ hard to solve and  we either give approximation guarantees or prove inapproximability results. We then introduce an elegant structure, namely that of \textit{intrinsic scores}, that makes many of these problems polynomial time solvable. We explicitly characterize the optimal groupings under this structure and show that the optimal solutions are related to \emph{homophilous} and \emph{heterophilous} partitions, well-studied in the psychology literature. For one of the four objectives, we show $NP$ hardness under the score structure and give a $\frac{1}{2}$ approximation algorithm for which no constant approximation was known thus far. Finally, under the score structure, we propose an online low sample complexity PAC algorithm for learning the optimal partition.  We demonstrate the efficacy of the proposed algorithm on synthetic and real world datasets.
\end{abstract}

\section{Introduction}
\label{sec:intro}
The problem of learning to partition users (or objects) into groups has numerous applications in ridesharing, health care groups, project groups etc. Effective grouping of users is critical in these applications as it determines how well the group can work together \cite{terborg1976longitudinal}. For instance, the success of ridesharing adoption depends on how the system assigns users to groups taking into account their source destination requirements as well as their personalities. A user who does not like making conversation would rate the ride-sharing experience low if they are in the same group as users who talk a lot. Here, the ridesharing company might be interested in maximizing the average \emph{happiness} of the rideshare groups. 
In another instance, effectively allocating employees to teams for a critical project might influence the project's success \cite{liang2007effect}. Here the organization might want to be defensive and maximize the compatibility of the least compatible pair of employees in any team.

That effective grouping leads to improved outcomes has been extensively studied in the psychology literature \cite{horwitz2005compositional}. In particular, the effects of \emph{homophily} (grouping similar individuals together) \cite{mcpherson2001birds,baccara2012homophily} and \emph{heterophily} (grouping dissimilar individuals together) \cite{stahl2010look} on desired outcomes have been documented in applications such as team formation \cite{horwitz2007effects} and study groups \cite{elbaum1999grouping,foorman2001critical} among others. While the social phenomemon of homophily and heterophily have been well studied \cite{lazarsfeld1954friendship} \cite{alpert1973optimal}, it is not clear (more so from a formal standpoint) as to when one should prefer homophilous to heterophilous grouping in general. 
    
\begin{table}
   \centering
         \begin{tabular}{|c|c|}
    \hline
\cellcolor{gray!25} \textbf{ (AoA) } & $\displaystyle \max_{\Pi(\S_1,\ldots \S_m)} \frac{1}{m}\sum_{i=1}^{m}H(\S_i|\W)$\\
\hline
\cellcolor{gray!25} \textbf{ (MoM) } & $\displaystyle \max_{\Pi(\S_1,\ldots \S_m)} \min_{i} \Big(  \min_{j,k \in \S_i} W_{jk} \Big)$ \\
\hline
\cellcolor{gray!25} \textbf{ (AoM) } & $\displaystyle \max_{\Pi(\S_1,\ldots \S_m)} \frac{1}{m} \sum_{i=1}^{m}  \Big( \min_{j,k \in \S_i} W_{jk} \Big) $ \\
\hline
\cellcolor{gray!25}~ & ~ \\
\cellcolor{gray!25} \textbf{(MoA) } & $\displaystyle
\max_{\Pi(\S_1,\ldots \S_m)} \min_{i} H(\S_i|\W)$ \\
    \hline
    \end{tabular}   
     \caption{Objectives considered.}
     \label{tab:objectives}     
\end{table}


In this work, we formalize the group formation problem by considering four concrete  objectives that involve pairwise compatibility of users. We refer to the pairwise compatibility as the \emph{happiness} of the pair of users. Under this happiness index, we study the following objectives:  maximize (i) the average of average \emph{happiness} across groups (AoA), (ii) minimum compatibility of any pair of users across  groups (MoM), (iii) average of the minimum compatible pairs across groups (AoM) and (iv) minimum of the average happiness across all groups (MoA) (see Table \ref{tab:objectives} for formal definitions). The objectives cover most of the scenarios one may be interested in maximizing. While we do consider general pairwise compatibilities, we introduce and focus mainly on an elegant structure that we impose on them. Under this structure, each individual has an associated \emph{intrinsic score} and the pairwise compatibility between two individuals depend on their corresponding scores. We will see that under this structure, the optimal solutions to these objectives naturally translate to homophilous and heterophilous groupings. The motivation of studying the score-based model, where the compatibility (or, preference towards 
each other) of a pair of users is determined by their 
relative scores, arises from the well-known Bradley-Terry-Luce (BTL) (Bradley and Terry, 1952 \cite{bradley}; Luce, 1959 \cite{luce}) in literature. The BTL model is often applied for pairwise 
comparisons, in order to determine the relative preferences. This led us to study the simpler case of score-based preferences, for which optimal polynomial 
algorithms can be provided in most cases, while the general case is either inapproximable or hard to approximate beyond a certain factor. 

We list below the major contributions of this paper.
\begin{center}{\textbf{Our Contributions} } \end{center}
\begin{itemize}
\item We show that all the objectives considered above are NP-hard (and hard to approximate beyond a factor) when we assume no structure on the pairwise compatibilities. We give polynomial algorithms under certain assumptions. 

\item When the compatibilities are dictated by intrinsic scores of each individual, we show that three of the four objectives become solvable in polynomial time. 

\item Under the intrinsic scores assumption, we explicitly characterize the structure of the optimal solution  and show that they are related to the notion of hompohily and heterophily.

\item We show that the MoA objective is NP-hard to solve even under the intrinsic scores assumption.

\item We propose a greedy algorithm for the  MoA objective and prove a $\frac{1}{2}$ approximation guarantee for the same. This is a significant improvement on known algorithms, since no constant factor result was known thus far. 

\item Under the intrinsic score structure, we propose a low sample complexity PAC learning algorithm to learn the scores (and hence the optimal groupings) provably. 
\end{itemize}

Table \ref{table:contributions} summarizes some of the contributions of this paper.

   \begin{table}
   \centering

    \begin{tabular}{|p{1.5cm}|p{2.4cm}|p{2.4cm}|p{2.4cm}|p{2.2cm}|}
        \hline
        ~&\textbf{AoA}&\textbf{MoM}&\textbf{AoM}&\textbf{MoA}\\\hline                     ~&~\cellcolor{blue!25}&~\cellcolor{blue!25}&~\cellcolor{blue!25}&~\cellcolor{blue!25}\\
        \textbf{General}&$1 - \frac{1}{{k\choose 2} m}$\cellcolor{blue!25}& \textsc{inapprox}\cellcolor{blue!25}& $\left(1 - \frac{1}{m}\right)$\cellcolor{blue!25}& $1 - \frac{1}{{k\choose 2}}$ \cellcolor{blue!25}\\
                ~&~\cellcolor{blue!25} & \cellcolor{blue!25}&~\cellcolor{blue!25}&~\cellcolor{blue!25}\\
                \hline
        ~&~&~&~&~\cellcolor{gray!25}\\
        \textbf{Score-based }&Homophily&Heterophily&Homophily&\cellcolor{gray!25} $\frac{1}{2}$\\
         ~&~&~&~&~\cellcolor{gray!25}\\
                \hline

    \end{tabular}
        \caption{Summary of main results of the paper. All shaded boxes correspond to NP-Hard problems. The top row (in blue) lists the hardness results, and the bottom row highlights the polynomial time solvability (unshaded) and approximation guarantee (gray).  \label{table:contributions}}
    \end{table}

\section{Related Work}
\label{sec:related_work}
Graph partitioning and clustering problems are well researched with several variants such as the disjoint clique problem, k-way partition problem, partition into triangles and many others \cite{feder1999complexity}. Balanced or capacitated versions of such partitioning or clustering problems are related to our work. For instance, the AoA objective is the same as the k-equipartition problem \cite{eppstein1991equipartitions} and is known to be NP-hard \cite{garey1979guide}. The \textsc{Sum-Max} partitioning problem \cite{watrigant2014sum} looks at minimizing the average of maximum weighted edge between pairs of groups and is close to, but not the same as the AoM objective. Min-max objectives similar to MoM and MoA objecives have been recently considered by Svitkina and Tardos  \cite{svitkina2004min} and Bansal et al. \cite{bansal2011min}, where the problem is to create a k-way equipartition to minimize the maximum weight of edges leaving a single group. This is different from the MoM objective where we want to maximize the minimum of minimum weights across edges within each group. 

Many seemingly related problems such as the clique partition problem \cite{grotschel1990facets}, the capacitated max-k-cut and min-k-cut problems \cite{gaur2008capacitated} do not come with size restrictions or have objectives that do not capture homophily or heterophily. In the clustering domain, the work closest to ours is that of Zhu, Wang and Li \cite{zhu2010data}, where heuristic algorithms for clustering with size constraints on the clusters are proposed, although for a majority of clustering applications such size constraints or information on the number of clusters is not pre-defined.

There is one class of graph partitioning problems that 
come with size restrictions, namely, the Graph Tree Partitioning Problem \cite{cordone}, where the objective is to 
partition a graph into equal size subsets, such that the 
weight of the spanning tree on each subset is either as high as possible (max-min) or as low as possible (min-max), or, the sum of the weights of the spanning trees 
are either as high (max-sum), or, as low (min-sum) as possible. Though these objectives are closely related to ours, they are not exactly the same problems. 

Some recent research in application areas such as team formation and ride-sharing study similar questions as us. Singla et al. \cite{singla2015} present online learning algorithms with PAC bounds in the context of learning the expertise of workers for team formation. The emphasis is on speeding up learning by exploiting similarities between workers and between tasks. In contrast to our work, their objective is to select an optimal subset of the workers after learning their expertise and there is no notion of forming groups. Bistaffa, Farinelli and Ramchurn \cite{bistaffa2014sharing} study an offline optimization problem of grouping users to minimize travel costs. This is formulated as a coalition formation problem, restricted by a social network based compatibility graph that is assumed given. The groups are  formed based on heuristics without any guarantees on optimality. Brindley, Blaschke and Walti \cite{brindley} look at what factors impact the formation of effective learning groups through an empirical study. 

Our formulations and learning algorithms can also be applied in recurring sharing economy settings (e.g., AirBnB) as well as healthcare. In the latter setting, it has been observed that assigning patients of similar disease characteristics to groups often helps in effective treatment \cite{dies1993research}. 


\section{Preliminaries}
\label{sec:prelims}

Let $[n] = \{1,2,\ldots,n\}$ be the set of items to be partitioned into groups. Let each group be of size $k$ and let $m = \frac{n}{k}$ denote the number of groups \footnote{Assume $k$ divides $n$ or add \emph{dummy} items if not.}. A $k$-\emph{partition} $\Pi(\S_1 \ldots \S_m)$ of $[n]$ is denoted by  a set of subsets $\{\S_1,\S_2 \ldots \S_{m}\}$ where each $\S_i \subseteq [n], |\S_i| = k$ with $\S_i \cap \S_j = \emptyset ~\forall i \neq j$ and $\bigcup_{i} \S_i = [n]$. We will capture the relation between users/objects using a symmetric \emph{pairwise compatibitlity matrix} $\W \in \R_+^{n \times n}$ where $W_{ij}(= W_{ji})$ denotes the compatibility of users/objects $i$ and $j$. Given a subset $\S \subseteq [n]$, we define the \emph{happiness} of the subset with respect to $\W$ as $H(\S|\W) = \frac{1}{|\S|^2}\sum_{i,j \in \S}W_{ij}.$ 
~\\
~\\
\textbf{Problem Definition: ~}\emph{Given a pairwise compatibility matrix $\W \in \bR_+^{n \times n}$, partition the $n$ items into $m$ groups in the \emph{best} possible manner that maximizes each of the  four objectives defined in Table \ref{tab:objectives}. }
~\\

Towards this, we will first consider the case where we don't impose any conditions on the pairwise compatibilities. In the subsequent section, we consider the same problem by imposing a score structure.


\section{General Compatibility Matrices}
\label{sec:general}
We start by describing some results on the hardness of approximation of the four objectives in the general case. 
\begin{thm}\textbf{(Approximability)}
\label{thm:MoM}
For $k\geq 3$, unless $P=NP$, the following are lower bounds on polynomial time approximability: 
(a) MoM: inapproximable, (b) AoM: $\left(1 - \frac{1}{m}\right)$,
(c) AoA: $1 - \frac{1}{{k\choose 2} m}$ and (d) MoA: $1 - \frac{1}{{k\choose 2}}$. 
\end{thm}

\begin{proof}
\textbf{MoM: inapproximable}.
Consider an instance $\mathcal{G} = \left(\mathcal{V}, \mathcal{E}\right)$ of \textsc{PartitionIntoTriangles} \cite{garey1979guide}, where $|\mathcal{V}| = 3q$ for some $q \in \mathbb{Z}^{+}$. The decision question is 
whether $\mathcal{V}$ can be partitioned in to $q$ disjoint sets of size $3$ each, 
${V_1, \ldots, V_q}$, such that each $V_i, i \in [q]$ forms a triangle in $\mathcal{G}$. 
Create an instance of MoM, a weighted graph 
$\mathcal{G}' = \left(\mathcal{V}', \mathcal{E}', \mathcal{W}'\right)$ 
from $\mathcal{G}$, where for every vertex $v \in \mathcal{V}$, 
we create a vertex $v'\in \mathcal{V}'$, and for every edge $e = (u,v) \in \mathcal{E}$, 
we create an edge $e' = (v',u') \in \mathcal{E}'$ between the corresponding vertices $v'$ and $u'$ 
in $\mathcal{V}'$, and set its weight $w_{e'} = M$, where $M$ is a large number. Set 
$k=3$ (group size) and $m = q$ (number of groups). 
For any pair of vertices $(p,q) \in \mathcal{V}$, such that no edge exists between $p$ and 
$q$ in $\mathcal{E}$, we create an edge $e'' = (p,q) \in \mathcal{E}'$ of weight $w_{e''} = \epsilon$, 
where $\epsilon$ is a small number and add it to $\mathcal{E}'$. 
If there exists a partition of $\mathcal{V}$ in to $q$ disjoint triangles, 
then there exists a solution to the MoM
problem with objective function $= M$. The $q$ disjoint triangles 
correspond to $m$ groups in $\mathcal{G}'$, where in each group, every edge 
has weight $M$, since the corresponding edge exists in $\mathcal{G}$. 
Similarly, if there does not exist any partition of $\mathcal{V}$ in to $q$ 
disjoint triangles, then every solution to MoM
on $\mathcal{G}'$ has value $=\epsilon$. This is because any solution 
to the group formation problem with $k=3$ and $m=q$ on $\mathcal{G}'$, 
would result in at least one group where at least one edge has a weight $\epsilon$, 
since the corresponding edge does not exist in $\mathcal{G}$.

Therefore, if there exists a polynomial time  approximation algorithm with an approximation 
ratio $> \frac{\epsilon}{M}$, one would be able to distinguish between 
the \textsc{Yes} and \textsc{No} instances of \textsc{PartitionIntoTriangles}, 
by creating a corresponding instance of MoM and applying the algorithm. 
A \textsc{Yes} instance would result in MoM $> \epsilon$, and 
similarly, a \textsc{No} instance would result in MoM $\leq \epsilon$. 
Hence, unless $P=NP$, there can be no polynomial time approximation algorithm with an 
approximation ratio $> \frac{\epsilon}{M}$. For $\epsilon \rightarrow 0$, 
and $M\rightarrow \infty$, we can make the ratio arbitrarily bad, $\rightarrow 0$. 
Hence, it is $NP$-hard to approximate the MoM problem in the general case. 

The approximability proofs of \textbf{MoA}, \textbf{AoA} and \textbf{AoM} use the same reduction as above. Specifically, create a graph $\mathcal{G}' = \left(\mathcal{V}', \mathcal{E}', \mathcal{W}'\right)$, with $k=3$ and $m=q$, from an instance of \textsc{PartitionIntoTriangles}, 
$\mathcal{G} = \left(\mathcal{V}, \mathcal{E}\right)$, where $|\mathcal{V}| = 3q$.  
For a \textsc{Yes} instance of the \textsc{PartitionIntoTriangles}, 
the disjoint $m$ groups in $\mathcal{G}'$ corresponding to the disjoint $q$ triangles in $\mathcal{G}$,  
will give a solution with MoA$=$AoA$=$AoM$=M$.
On the other hand, for a \textsc{No} instance, any partition of $\mathcal{G}'$ 
into $m$ disjoint groups would result in at least one group, with at least one 
edge of weight $\epsilon$. Setting $\epsilon=0$, therefore, MoA $\leq \frac{2M}{3}$, AoA $\leq M\left(1 - \frac{1}{3m}\right)$, and
AoM $\leq M\left(1 - \frac{1}{m}\right)$. 

Therefore, if there exists a polynomial time approximation algorithm for any of the above three objectives with better approximation factors (MoA:$\frac{2}{3}$, AoA:$\left(1 - \frac{1}{3m}\right)$, and 
AoM: $\left(1 - \frac{1}{m}\right)$), one can distinguish between the \textsc{Yes} and and \textsc{No} instances of \textsc{PartitionIntoTriangles}. 
The approximation guarantees can be extended to general $k>3$, by similar reduction, replacing triangles by 
$k$-cliques. 
\end{proof}



\subsection{Polynomial Solvability for $k=2$}
While the general case is hard, for the case where $k=2$, all the objectives become polynomial time solvable.
\begin{theorem}\textbf{(Polynomial solvability for $k=2$)}
\label{thm:k=2}
When $k=2$, all four objectives are polynomial time solvable.
\end{theorem}
For proving Theorem \ref{thm:k=2}, we first prove the 
following claim. 
\begin{claim}
\label{cl:AoA=AoM}
When $k=2$, the optimal solution for MOA is the same 
as that of MOM, and the optimal solution for AoA is the same as the optimal solution for AoM. 
\end{claim}
\begin{proof}
We first prove that the optimal solution for MoA is the same as that 
for MoM when $k=2$. Since group sizes are $k=2$, every group has only one edge, that occurring between the pair of vertices in the group. 
Therefore, the average of the weight of edges in any group is determined simply by the 
weight of the single edge in the group. Hence, the minimum of the averages, namely MoA 
is the same as the minimum weight edge in any group, that is, MoM.
We next argue that the optimal solution for AoA is the same as that for AoM when $k=2$. Since every 
 group has only a single edge each, the solution maximizing the average of 
 the average weight of every group, namely, AoA, is the same as maximizing the average of the weight of the edge in every group, which corresponds to AoM for $k=2$. 
\end{proof}

Now, we discuss how Edmond's maximum weighed matching 
for general graphs solves AoA (and AoM) optimally in polynomial time for $k=2$.
We find a maximum weight matching in $\mathcal{G}$ using Edmond's algorithm \cite{edmonds}. 
This returns a partition of $\mathcal{V}$ in to $\frac{n}{2}$ groups, maximizing 
the total weight of edges used in the matching. Suppose there is an optimal 
solution for AoA (also, AoM) for $k=2$ in $\mathcal{G}$, that has a higher 
objective function value. That means that there exists a partition of 
$\mathcal{V}$ in to $\frac{n}{k}$ disjoint subsets, with each 
subset of size $2$, hence including only one edge, such that the total weight 
of edges used in the subsets is higher than that returned by Edmond's algorithm. 
However, that contradicts the optimality of Edmond's algorithm, since we can use 
the solution to AoA (also, AoM) to find a higher weight matching. 
Similarly, suppose there is a higher weight matching than the objective function 
value of the optimal solution for AoA (also, AoM). Then, one can use the matching 
returned by Edmond's algorithm to find a higher value, feasible solution for AoA (also, AoM), 
thereby contradicting the optimality of the optimal value for AoA (also, AoM). 

For MoA and MoM, we first prove the following 
property. 
\begin{lemma}
\label{lemma:size2}
Given a graph $\mathcal{G}$, and an 
optimal solution value $OPT$ for MoM (also, MoA), if we delete 
all edges of weight $<OPT$, and make the resultant graph $\mathcal{G}'$ unweighted, 
then there exists a perfect matching in $\mathcal{G}'$, corresponding to an optimal 
solution in $\mathcal{G}$. 
\end{lemma}
\begin{proof}
Consider the groups in an optimal solution for MoM (also, MoA) 
in $\mathcal{G}$. Since the optimal solution has value $OPT$, 
there must exist a partition of the vertices in to groups of $2$, 
such that the edge in each group has a weight $\geq OPT$. 
However, this corresponds to a matching in $\mathcal{G}$, 
where every edge in the matching has a weight $\geq OPT$. 
Now, delete every edge of weight $<OPT$. The matching 
remains unperturbed, and corresponds to a perfect 
matching in $\mathcal{G}'$.
\end{proof}
Now, we give the algorithm for optimally solving MoM and MoA for $k=2$. 
The algorithm for maximizing MoM (also, MoA) would involve ordering the distinct weights in $\mathcal{W}$ 
in non-decreasing order. For the next weight $w$ in the list, delete all edges 
of weight $w_e< w$. Create an unweighted graph $\mathcal{G}'$ 
containing only the remaining edges in $\mathcal{G}$ (without 
considering their weights). Now, use Edmond's algorithm to find a maximum 
cardinality matching in $\mathcal{G}'$. The lowest weight $w$, after deleting which, 
$\mathcal{G}'$ does not have a perfect matching, is the optimal value 
for MoM (also, MoA). 
The optimality of this algorithm follows from Lemma \ref{lemma:size2} and the optimality of Edmond's maximum cardinality matching for general graphs, that runs in polynomial time. The polynomial time solvability also follows, hence. 

We next study a linear time solvable special case. 

\subsection{Transitive Compatibility: Optimal Linear Algorithm for any \textit{k}}
We next prove that under a transitivity like assumption on the compatibility matrix, there exists a linear time optimal algorithm. The assumption follows 
from the intuition that if user $i$ is compatible with $j$, $j$ is compatible with $k$, then $i$ is compatible with $k$.  
Formally, the transitive property assumed is: 
$\forall {i,j,k}, W_{ij} \geq \min {\left(W_{ik}, W_{kj}\right)}$. 
The following theorem follows from the fact that graphs obeying transitive compatibility would have a particular 
structure: a collection of disjoint cliques, and, a linear traversal of the graph 
would return the optimal solution. 

\begin{thm} 
\label{thm:trans}
Under the transitive compatibility property in $\mathcal{G}$, there exists a 
linear time optimal algorithm for MoM.
\end{thm}
\textbf{Proof Sketch:}\\
We first argue that the graph $\mathcal{G}$ on which the transitive compatibility property holds 
has a certain structure. Specifically, it consists of disjoint connected components, 
where each connected component is a clique. 
To see this, consider a pair of vertices in $\mathcal{G}$, $u$ and $v$, between whom the 
happiness or compatibility is $0$. In other words, no edge exists between them. We argue that in order to maintain the 
transitive property, 
any vertex $p$ that $u$ is adjacent to (i.e., $W_{u,p} >0$),  must necessarily have a $0$ compatibility with $v$. 
This is because if there exists an edge between $p$ and $v$, then with both $W_{p,v} >0$ and $W_{u,p}>0$, the 
transitive property is violated by 
$0= W_{u,v} < \min{\left(W_{u,p}, W_{p,v}\right)}$. Therefore, if $W_{u,v} = 0$, then $W_{p,v} = 0$ 
for all $p$, such that $W_{u,p} >0$. 
Similarly, any vertex $p'$ that $p$ is adjacent to cannot have any edge to $v$. Continuing in this manner, 
it can be seen that any vertex that $u$ is connected to cannot have any edge to $v$. Therefore, $u$ and 
$v$ must be in disjoint connected components. 
Similarly, any pair of vertices that $u$ is adjacent to, say, $v_1$ and $v_2$, must be adjacent to each other. Otherwise, it can be seen that the transitive compatibility 
property is violated. 

Another property that transitive compatibility induces is that, in every clique, there can be at most one edge of higher weight, and all other edges must be of identical weight. To see this, note that if there are two edges of higher weight compared to the weight of all other edges in a clique, then there would be at least one triangle, where one edge is lower in weight than the other two, violating the transitive property. 

Any optimal solution would consider 
each clique separately (in other words, only the participants belonging to the same clique would be matched 
to one another), since otherwise the partitions 
would include edges of weight $=0$. 
Moreover, in every grouping of the clique vertices, there can be at most one group with at most one edge of higher weight, and all other edges in all the groups would have identical lower weight, say $w$. Hence, replacing the higher weight edge by an edge of weight $w$ would not change the MoM objective value. Hence, any partitioning of the vertices of a clique in to groups of size $k$ \footnote{Without loss of generality, for any clique of size $n'$, we add dummy vertices with edges of weight $w$ incident on them, to make its cardinality, a multiple of $k$.} would be optimal. 
The algorithm is linear in the number of vertices, since one has to traverse 
at most all the vertices in all the cliques in order to get the partitions.

\section{Score Based Compatibility Matrices}
\label{sec:score}
In this section we consider a simple yet useful structure on the pairwise compatibility matrix under which three out of the four objectives introduced in Section \ref{sec:prelims} become poly-time solvable. Specifically, we consider the case of \emph{score based compatability} matrices where every item $i$ has an associated score $ s_i \in \R_+$ and the pairwise compatibility of items $i$ and $j$ is given by the product of their individual scores $s_is_j$. This is a natural assumption in several practical scenarios. For instance in study groups, the score could refer to the influence a student has on her peers which may depend on the GPA of the student. In ride  sharing applications the score may indicate how introverted/extroverted a person is. 

We begin by defining certain natural partitions induced by score vectors. 
\begin{defn} \textbf{(Homophilous Partition)}
\label{defn:homophily}
Let $\s \in \bR^n$  and let $\sigma = \argsort(\s)$\footnote{$\argsort(\s)$ is the permutation obtained by sorting the values of score vector $\s$ in non-increasing order. Specifically, for $\sigma=\argsort(\s)$, for any $i,j$, $s_i > s_j \implies \sigma(i) < \sigma(j)$.}. A $k$-partition $\Pi(\S_1,\ldots,\S_m)$ corresponding to  $\s$ is called \emph{homophilous} w.r.t $\s$ if $\forall i \in [m], \S_i = \{ \sigma((i-1)k+1, \ldots, \sigma((i-1)k + k) \}$
\end{defn}

\begin{defn}\textbf{(Heterophilous Partition)}
Let $\s \in \bR^n$ and let $\sigma = \argsort(\s)$. A $k$-partition $\Pi(\S_1,\ldots,\S_m)$ corresponding to  $\s$ is called \emph{heterophilous} if $\forall i \in [m],  \S_i = \{\sigma((i-1)(k-1) + 1), \ldots , \sigma((i-1)(k-1) + (k-1)), \sigma(n+1-i))\}$
\end{defn}

As an example, let $\sigma = (1~2~3~4~5~6)$ and $k=2$. The homophilous $2$-partition corresponding to $\sigma$ would be $\{\S_1 = (1~2), \S_2 = (3~4), \S_3 = (5~6)\}$ whereas the heterophilous $2$-partition would be $\{\S_1 = (1~6), \S_2 = (2~5), \S_3 = (3,4)\}$. Our main results of this section explicitly characterize the optimal solutions for the objectives considered.
\begin{thm}\textbf{(Homophilous Partition is Optimal for AoA and AoM)}
\label{thm:homophily}
Let $\W \in \bR^{n \times n}$ be a score based compatibility matrix parametrized by the score vector $\s \in \bR^n$. The optimal solution to the average of averages (AoA) and the average of minimums (AoM) objectives w.r.t $\W$ is given by the \emph{homophilous} partition of $\s$.
\end{thm}

\begin{proof}
Note that for any group $\S$, if the compatibility matrix is score based, then the sum of weights of all pairs in the group is given by 
$\sum_{i,j \in \S} W_{ij} = \Big(\sum_{i \in \S} s_i\Big)^2$
~\\
\textbf{AoA Objective:}
Assume wlog the entries of $\s$ are sorted in descending order i.e., $s_i > s_j$ for all $i < j$. 
Let $\Pi$ denote the homophilous $k$-partition corresponding to $\s$. For the sake of contradiction let $\bar{\Pi} \neq \Pi$ be the optimal $k$-partition. Let $i$ be the minimum index such that both $i$ and $i+1$ are in the same groups in $\Pi$ whereas they are in different groups in $\bar{\Pi}$. Denote these groups by $g_1$ and $g_2$. We will show that by swapping specific elements from $g_1$ and $g_2$, one can obtain a partition which is at least as good as $\bar{\Pi}$. Let $i \in g_1$ and $i+1 \in g_2$. Notice that there must be at least one element $x \in g_1$ such that $s_{i+1} > s_x$ (otherwise it contradicts the minimality of $i$). We will consider two cases depending on whether $\sum_{j \in g_1} s_j > \sum_{j \in g_2} s_j$ or otherwise.
Case 1: $\sum_{j \in g_1} s_j > \sum_{j \in g_2} s_j$. Let $v_1 = \displaystyle \sum_{j \in g_1; j \neq i,x} s_j$, $v_2 = \displaystyle \sum_{j \in g_2; j \neq i+1} s_j$. Then, using convexity of $H(\S_i|\W)$, we have
$$(s_i + s_{i+1} + v_1)^2 + (s_{x} + v_2)^2 \ge (s_i + s_{x} + v_1)^2 + (s_{i+1} + v_2)^2$$ 

Case 2: $\sum_{j \in g_1} s_j \le  \sum_{j \in g_2} s_j$
In this case, there must be an element $x \in g_2$ such that $s_x < s_i$. One can follow a similar proof as the previous case by swapping $i \in g_1$ and $x \in g_2$.
In both the cases, we obtain a partition whose sum (average) of weights  over the groups of the partition is at least as good as $\bar{\Pi}$. One can repeat the procedure with the new partition obtained until one reaches $\Pi$. But this contradicts the fact that $\bar{\Pi} \neq \Pi$ is optimal. 
~\\
\textbf{AoM Objective:}
Assume wlog that the score vector $\s$ is such that $s_1 \le s_2 \ldots \le s_n$. 
We will show the result using induction on the number of groups $m$. Consider the base case where $m=2$ i.e. $n = 2k$. In this case, we need to show:
$$s_1s_2 + s_{k+1}s_{k+2} \ge s_{i_1}s_{j_1} + s_{i_2}s_{j_2}$$
where $(i_1,j_1)$ and $(i_2,j_2)$ correspond to the minimum compatibile pairs in the two groups corresponding to some non-homophilous partition of $\s$. If any of these pairs is same as $(1,2)$ then the result is obvious. Assume not. Then both $1$ and $2$ will contribute to the minimum compatible pairs. Thus it is enough to show that both of the below cases hold
$$s_1s_2 + s_{k+1}s_{k+2} \ge s_1s_{k+2} + s_2s_3 ~\text{and}~$$
$$s_1s_2 + s_{k+1}s_{k+2} \ge s_1s_{3} + s_2s_{k+2}$$
as all other cases give rise to smaller objective values. But
\begin{eqnarray*}
s_{k+2}(s_3 - s_2) & \ge & s_1(s_3 - s_2) \\
\implies s_1s_2 + s_{k+1}s_{k+2} & \ge & s_1s_{k+2} + s_2s_3
\end{eqnarray*}
Similarly one can show the result for the other case as well. This proves the base case. Now for a general $k$, assume that the induction hypothesis is true for $m = k-1$. For $n = mk$, apply the induction hypothesis to the bottom $n - k$ items i.e, for the set $\{s_{k+1} , \ldots s_n\}$. We need to show that with the newly added items $\{s_1,s_2,\ldots,s_k\}$ the hypothesis is still satisfied. Assume not. Then there must exist index pairs $(i_1,j_1),\ldots (i_m,j_m)$ such that 
\begin{eqnarray*}
& & s_1s_2 + s_{k+1}s_{k+2} + \ldots s_{(m-1)k+1}s_{(m-1)k+2}  <  s_{i_1}s_{j_1} + s_{i_2}s_{j_2} + \ldots s_{i_m}s_{j_m}.
\end{eqnarray*}
If items $1$ and $2$ are in the same group, we  arrive at a contradiction. Assume they are in different groups. Let $(1,j_1)$ and $(2,j_2)$ be the corresponding minimum pairs. By swapping $2$ with $j_1$, we can only increase the objective. We can iteratively swap items without decreasing the objective such that the first $k$ items are in the first group. But this contradicts the induction hypothesis for the last $n-k$ items. 
\end{proof}

\begin{thm}\textbf{(Heterophilous Partition is Optimal for MoM)}
\label{thm:heterophily}
Let $\W \in \bR^{n \times n}$ be a score based compatibility matrix parametrized by the score vector $\s \in \bR^n$. The optimal solution to the minimum of minimums (MoM) objective w.r.t $\W$ is given by the \emph{heterophilous} partition of $\s$.
\end{thm}
\begin{proof}
We prove the base case for $n=4$ and $k=2$. Let the scores be given by $s_1 > s_2 > s_3 > s_4$. The three possible ways of partitioning this are given by $\{(1, 2)(3, 4)\}, \{(1, 3)(2, 4)\},$ $ \{(1, 4)(2, 3)\}$. Note that we have  $s_1s_4 > s_3s_4$ and $s_2s_3 > s_3s_4$. Thus, 
	 $$\min(s_1s_4,s_2s_3) > s_3s_4 = \min(s_1s_2,s_3s_4)$$
	 Similarly, we have $s_1s_4 > s_2s_4$ and $s_2s_3 > s_2s_4$. Thus,
     $ \min(s_1s_4,s_2s_3) > s_2s_4 = \min(s_1s_3,s_2s_4)$.

As the induction hypothesis, assume that the claim is true for some $n$ and $k$. We will show that it is true for $n+k$. Let the new items added be $x_1,x_2,\ldots x_k$ assume wlog that 
	 $x_1 > x_2 > x_{k-1} > s_1 > s_2 > \ldots s_n > x_k$.
We know from the induction hypothesis that the heterophilous partition corresponding to $\s$ is has the highest MoM objective value. For the sake of contradiction, assume that by adding the new $k$ items, the heterophilous partition corresponding to the vector $[x_1 ~x_2 ~\ldots x_{k-1} ~s_1~ s_2 \ldots s_n~ x_k]$ is not optimal. 

Case 1: As $x_1, \ldots x_{k-1}$ are larger than $s_i ~\forall i$ and $x_k$, the only way any of them could be a part of the minimum pair is when all the items with scores $x_1, \ldots x_{k-1}$ are in the same group. If $x_k$ is also in this group, we  arrive at a contradiction. If item with score $x_k$ is not in this group, then we swap the item with some item with score $s_p$ to arrive at a contradiction.  

Case 2: The other case to consider is when none of the items with scores $\{x_1,\ldots x_{k-1}\}$ contribute to the minimum compatible pair in their respective groups. In this case, we start with the group $g$ that contains at least one item with scores of $\{x_1 , \ldots x_{k-1}\}$ and has the highest minimum compatible weight. We iteratively swap the remaining items with scores $x_i$ for some $i$ which are not in $g$, with items from $g$ with score $s_j$ for some $j$. This can be done without decreasing the MoM objective value. We then continue swapping iteratively to get to the partition where the group $g$ consists of items with scores $\{x_1,\ldots,x_{k-1}\}$ along with some other item with score $s_p$ for some $p$. Now, we can use the argument from case $1$ to arrive at the required contradiction.
\end{proof}

\begin{thm}\textbf{(Hardness of MoA)}
\label{thm:MoA-hardness}
Computing the optimal partition for the minimum of averages (MoA) objective for a score based compatibility matrix is NP-hard.
\end{thm}
\begin{proof}
Consider an instance $\mathcal{I}$ of \textsc{3-Partition}, with $n = 3m$ items, each associated 
with a value; $s_i$ is the value for the item $i$. The total sum of the 
values is $\sum_{i \in \mathcal{I}}{s_i} = mB$, and the size of each item 
is $\frac{B}{4} < s_i < \frac{B}{2}$. The decision problem is whether there 
exists a partition of the items into $m$ partitions, such that 
the sum of the values of items in each partition is exactly $B$. This is a strongly NP-hard problem. 

Now, construct an instance $\mathcal{I}'$ of the MoA problem, with $k=3$, where we create an item in 
$\mathcal{I}'$ 
corresponding to every item in $\mathcal{I}$, and the score associated with the item 
in $\mathcal{I}'$ is set to the value of the 
corresponding item in $\mathcal{I}$. In the MoA problem our 
goal is to partition the items in $\mathcal{I}'$  
in to $m$ groups, each containing exactly $k=3$ items, such that the total score of each group is as high as possible. In fact, the lowest total score determines MoA. The decision question 
we ask here is as follows: does there exist a partitioning of the $n$ items into $m$ 
groups, each group containing exactly $k=3$ items, such that the total score of each group 
is at least $B$?
 
If there exists a partition of the items 
into $m$ groups, such that every group contains exactly $3$ items, 
and the total score of every group is $\geq B$, then that 
corresponds to a \textsc{YES} instance of 
\textsc{3-Partition}. Clearly, every group has to sum up to exactly $B$, since the total sum 
is $mB$. Alternately, if $\mathcal{I}$ corresponds to a \textsc{YES} instance of 
\textsc{3-Partition} for a given value $B$, then note that the corresponding $m$ partitions 
would have exactly $3$ items because of the choice of the range of the values 
of the items (and hence, the range of the scores), and each 
partition would sum up to $B$. This would give a feasible solution for the 
MoA in $\mathcal{I}'$, with each group's score summing up to $\geq B$. 
This completes the reduction. 
\end{proof}

While MoA is NP-Hard in general, we give a simple algorithm \textsc{Greedy} for the MoA objective given 
$\W \in \bR^{n \times n}$, a score based compatibility matrix: Sort the objects/items by their scores in a non-increasing order. Take the next unassigned item from 
the list and assign it to the partition (or group) with the lowest total score thus far, as long as the partition is not full 
(i.e., it has $<k$ vertices). Break ties arbitrarily. The best known approximation factor for this algorithm is $\max{\left(\frac{2}{k}, \frac{1}{m}\right)}$ (\cite{he}). We prove a constant factor approximation below, a significant improvement.
\begin{thm}\textbf{(Greedy Algorithm is $\frac{1}{2}$ approx for MoA)}
\label{thm:MoA-approx} Algorithm \textsc{Greedy} produces a $k$-partition that is a $\frac{1}{2}$ approximation for the MoA objective for score based compatibilities.
\end{thm}
\begin{proof}
For the first $m$ 
iterations, each group will receive one item each from the top of the sorted list. 
Define a Reduced Set $\mathcal{R}$ as a set of groups in the final solution obtained by the 
greedy algorithm, such that:
(a) for any $(p,q) \in \mathcal{R}$, there is at least one iteration $t$ of 
the greedy algorithm after the first $m$ iterations when $p$ 
is favored over $q$ for assigning the next available item, while $q$ is not 
full, and there is also at least one iteration $t'>m$ when $q$ receives an item while $p$ is not full, 
(b) $\mathcal{R}$ has the maximum cardinality among all such sets of groups, 
(c) $\mathcal{R}$ includes the group $j\in [m]$ that receives the $m^{th}$ item in the sorted 
list as its first item. 

Let $v_{j,i}$ be the $i^{th}$ item added to the $j^{th}$ group with size $s_{j,i}$, and the earliest iteration by which all groups in $\mathcal{R}$ get full be $t_R$. Then it  follows from the definition of $\mathcal{R}$:
\begin{observation}
\label{obs:1}
Any group $j\notin \mathcal{R}$ would receive all the items from 
from the second one to the $k^{th}$ one, that is, items $\{v_{j,2}, \ldots, v_{j,k}\}$ 
in iterations $t>t_R$. 
\end{observation}
 
This follows from the definition of $\mathcal{R}$. Suppose a group $j' \notin \mathcal{R}$ received 
the item $v_{j',2}$ in an iteration $t'<t_R$. In that case, 
there exists at least one iteration, specifically, $t'$, when $j'$ is favored over each of the groups in 
$\mathcal{R}$, and at least one iteration $\leq t_R$, when each of the groups in $\mathcal{R}$ 
are favored over $j'$, since by $t_R$ all groups 
in $\mathcal{R}$ get full. This implies, that $j'$ should have been included in $\mathcal{R}$, 
and $\mathcal{R}$ is not a maximal set. 
 


\begin{claim}
\label{cl:1}
Let the items in the sorted list be $\{v_1, v_2, \ldots, v_n\}$, 
where $s_1 \geq s_2 \geq \ldots s_n$, ($s_i$ is the score of item $v_i$, the 
$i^{th}$ item in the sorted list). Let $m'=m-|\mathcal{R}|$. 
The items $v_{m'+1}, v_{m'+2}, \ldots, v_{m'+|\mathcal{R}|k}$ 
get assigned to groups in $\mathcal{R}$. 
\end{claim}

\begin{proof}
Suppose an item in the sequence $v_{m'+r}$, $r \in \{1, \ldots, |\mathcal{R}|k\}$  
gets assigned to a group $j \notin \mathcal{R}$. 

Case 1: $v_{m'+r}, r\geq 1$ is the first item assigned to $j$. Now, there is at least one group $j' \in \mathcal{R}$, that 
receives an item  earlier in the list than $v_{m'+r}$ as its first item as the greedy algorithm
 assigns one item each to each of the groups, 
before assigning the second item to any group, and there are only $m'$ groups outside $\mathcal{R}$. 
Thus, once $j'$ receives the second item, the total score of $j'$ 
will be  greater than that of $j$, and hence, $j$ will receive at least 
one item before $j'$ is chosen again. But from Observation \ref{obs:1},  groups $\notin \mathcal{R}$ do not receive their second item till 
all groups in $\mathcal{R}$ are full. 

Case 2: $v_{m'+r}$ is the second item assigned to $j \notin \mathcal{R}$. 
Since from Observation \ref{obs:1}, $j$ receives its 
second item only after $t_R$,  and all groups receive at least 
one item before any group receives its second item, 
this is possible only if $m'+r > |\mathcal{R}|k + m'$. However, 
$m'+r \leq |\mathcal{R}|k + m'$ by assumption. Hence this is not possible either. 
\end{proof}

From Claim \ref{cl:1}, it follows that the $m'$ groups $\notin \mathcal{R}$, receive 
the $m'$ largest score items as their first items, specifically, items 
$v_1, \ldots, v_{m'}$. Let us call these first $m'$ items as \emph{large} items, 
since the groups receiving them do not get their second item till all groups 
in $\mathcal{R}$ are full. 
Moreover, it also follows that any subset of $|\mathcal{R}|k$ items from 
$\mathcal{V} \setminus \{v_1, \ldots, v_{m'}\}$ (that is, from the set of all items 
excluding the large items), would have a total score at most the total 
score of the items used to fill $\mathcal{R}$. This can be seen 
from the fact that $\mathcal{R}$ gets the highest score $|\mathcal{R}|k$ items, 
excluding the large items. 

\begin{claim}
\label{cl:2}
Denote the sum of scores of a set of items $\mathcal{S}$ as $Sum(\mathcal{S})$. Then,
$$\max_{\left(\left(\mathcal{S}\subset \{\mathcal{V}\setminus \{v_1, \ldots, v_{m'}\}\}\right) \cap \left(|\mathcal{S}| = |\mathcal{R}|k\right)\right)}{Sum(\mathcal{S})} \leq \sum_{j \in {\mathcal{R}}}{\sum_{i \in [k]}{{s_{i,j}}}}.$$ 
\end{claim}

\begin{proof}
This follows from Claim \ref{cl:1}. Therefore, 
 $\mathcal{R}$ gets the highest score $|\mathcal{R}|k$ items, excluding the large items. 
\end{proof}
%

\begin{claim}
\label{cl:3}
 Let $OPT$ be the optimal value of MoA for a given instance. Then
$OPT \leq \frac{\sum_{j \in {\mathcal{R}}}{\sum_{i \in [k]}{{s_{i,j}}}}}{|\mathcal{R}|}$
\end{claim}

\begin{proof}
There is a partition of the items into $m$ groups, 
such that the sum of the scores of each group is $\geq OPT$. Let $AVG_{\mathcal{R}} = \frac{\sum_{j \in {\mathcal{R}}}{\sum_{i \in [k]}{{s_{i,j}}}}}{|\mathcal{R}|}$. 
If $|\mathcal{R}|=m$, then the claim is obvious.  
Suppose for contradiction, that $OPT> AVG_{\mathcal{R}}$ when $1 \leq |\mathcal{R}| <m$. 
We have $m' = m - |\mathcal{R}|$ groups outside the reduced set. 
Therefore, we only have $m'$ large items, that can be distributed to at most 
$m'$ groups in any optimal solution. 
The remaining $\geq |\mathcal{R}|$ groups in the optimal solution (that do not 
receive any large item) would each need to get 
$k$ items, summing up to $\geq OPT$. Therefore, there must exist 
a subset of items, say $\mathcal{S}$, of cardinality $|\mathcal{R}| k$, excluding 
the large items, such that their score sums 
up to $\geq |\mathcal{R}| OPT$. 
From Claim \ref{cl:2}, therefore, 
$AVG_{\mathcal{R}}|\mathcal{R}| \geq Sum(\mathcal{S}) \geq |\mathcal{R}| OPT$. 
However, this contradicts the assumption that $AVG_{\mathcal{R}} < OPT$. 
This completes the proof. 
\end{proof}

Hence, we conclude from Claim \ref{cl:3} that 
$OPT \leq AVG_{\mathcal{R}}$. 
Suppose the minimum total score is realized by a group $r \notin \mathcal{R}$. 
Clearly, $s_{r,1} \geq \sum_{p \in [k-1]}{s_{q,p}}\ \forall q \in \mathcal{R}$. 
Since we assign the items in non-increasing order of their scores, 
$s_{q,k} \leq \frac{\sum_{p \in [k-1]}{s_{q,p}}}{k-1} \ \forall q \in \mathcal{R}$, hence, 
$s_{r,1} \geq \left(1 - \frac{1}{k}\right)\frac{\sum_{j \in {\mathcal{R}, i \in [k]}{s_{i,j}}}}{|\mathcal{R}|}$. Therefore, $s_{r,1}\geq \left(1 - \frac{1}{k}\right)OPT$. 
Hence, for $k\geq 2$, the realized minimum total score in this case is $\geq \frac{1}{2} OPT$. 

Now, consider the case when the minimum total score is realized by some group in $\mathcal{R}$. 
Let $p$ and $q$ be the groups with 
the lowest and highest sum of scores in $\mathcal{R}$ respectively. 
By definition of $\mathcal{R}$, there exists at least one iteration after all the groups 
have received one item each (that is, some iteration $>m$), when the total score of $q$ 
was lower than $p$, and hence $q$ got assigned an item favored over $p$.
Let $k' \in [2, \ldots, k]$ be the highest index such that the $q$ received the $k'^{th}$ item
while $p$ was not full. Let $p$ have $k''<k$ items assigned at that time. 
Since $q$ was favored over $p$, $\sum_{i\in [k'']}{s_{p,i}} \geq \sum_{j\in [k'-1]}{s_{q,j}}$. 
Also, $s_{q,k'} \leq s_{p,k''}$. 
After this, $p$ received $k-k''$ more items, before $q$ received 
any of the items in $\{v_{q,k'+1}, \ldots, v_{q,k}\}$, if $k'<k$ (otherwise, $q$ 
would not receive any items after this). 
Clearly, $s_{q,\ell} \leq s_{p,k}$ for $\ell \in \{k'+1, \ldots, k\}$. 
Hence, the total score of $q$ is 
$\sum_{i\in[k]}{s_{q,i}} \leq \sum_{i\in [k'']}{s_{p,i}} + s_{p,k''} + (k-k')s_{p,k}$. 
Therefore, 
$\sum_{i\in[k]}{s_{q,i}} \leq \sum_{i\in [k'']}{s_{p,i}} + s_{p,k''} + 
(k-k'')s_{p,k} + (k''-k') s_{p,k}$
$\leq S_p + s_{p,k''} + (k''-k') s_{p,k}$. 
Now, $k''-k'< k$, and because we consider items in sorted order, 
$s_{p,\ell} \geq s_{p,k} \forall \ell \in [k]$. Hence, 
$s_{p,k''} +  (k''-k') s_{p,k} \leq s_{p,{k''}} + \sum_{i \in \{1, \ldots, k\} \setminus k''}
{s_{p,i}} \leq S_p$. 
Hence, $S_q \leq 2 S_p$. Since $S_q \geq AVG_{\mathcal{R}} \geq OPT$, the realized minimum total score $S_p \geq \frac{OPT}{2}$. 
This concludes the proof. 
\end{proof}

\begin{algorithm}[H]
\begin{algorithmic}
\label{alg:ordLearn}
\STATE \textbf{Paramters:} Number of users $n$, number of groups $m$, group size $k$, confidence $\delta$
\STATE Set $\delta^*$ as in Theorem \ref{thm:partitionLearn-score}
\STATE Generate a Erdos-Renyi random graph $G \sim \mathcal{G}(n,\frac{\log(n)}{n})$
\STATE Let $diam(G)$ be the diameter of $G.$ 
\STATE Let $E_1,\ldots E_{\ell}$ be a partition of the edges of $G$ into $\ell$ bins got using an (approximate) minimum edge coloring $G$.
\FOR {$i$ = $1$ to $\ell$}
\STATE \quad Divide $E_i$ arbitrarily into $b_i := \lceil{|E_i|/m}\rceil$ disjoint bins $\{B^i_1,\ldots B^i_{b}\}$.
\FOR{ $j = 1: b_i$}
\STATE  Play $k$-partitions corresponding to $B^i_{j}$ for $O\Big(\frac{diam(G)^2}{\Delta^2}\ln(\frac{1}{\delta^*})\Big)$ rounds each.
\STATE  Estimate $\hat{s}_p - \hat{s}_q$ for all edges $(p,q) \in B^i_j$
\ENDFOR
\ENDFOR
\STATE For each $k$, estimate $\hat{s}_k$ by summing the estimates for $\hat{s}_i - \hat{s}_j$ along the shortest path in $G$ from $1$ to $k$. If no path exists for node $k$, set $\hat{s}_k$ = 0;
\STATE Return $\hat{\sigma} = \argsort(\hat{\s})$.
\caption{\textsc{LEARNORDER}}
\end{algorithmic}
\end{algorithm}

\section{Learning Score Vector and Guarantees}
\label{sec:Learning}
In this section, we propose an algorithm for adaptively learning the optimal ordering corresponding to the score vector of a pairwise compatibility matrix. The learning proceeds in rounds and for every group $\S^t_i$ in a chosen $k$-partition $\Pi(\S^t_1,\ldots, \S^t_m)$ at round $t$, we assume that we receive a iid noisy version of the happiness of the group as the response i.e. $H(\S^t_i) + \eta^t_i$ where $\forall t,i, ~\eta^t_i \in [-b,b]$ for some $b > 0 $ and $\mathbb{E}(\eta^t_i) = 0$. 

Our goal is to learn the ordering corresponding to the score vector $\s \in \bR^n$ of the pairwise compatibility matrix $\W \in \bR^{n \times n}$, by choosing groups adaptively for $T(n,k,\delta)$ rounds. Here $k$ is the size of groups chosen in each round, and $\delta$ is the failure probability for learning a wrong ordering. 
Once the ordering is learned, we can compute the optimal (or approximately optimal) partition  for the various objectives by sorting the scores and invoking Theorems \ref{thm:homophily}, \ref{thm:heterophily} and \ref{thm:MoA-approx}. 



The algorithm to learn the ordering is given in Algorithm $1$. The Algorithm \textsc{LearnOrder} begins by generating a random Erdos-Renyi graph $G$ where the probability of an edge being present is $\frac{\log(n)}{n}$. Thus the expected number of edges in the graph is $n\log(n)$. The edges of $G$ are then partitioned into disjoint pieces using an approximate $O(\Sigma)$ edge coloring where $\Sigma$ is the maximum degree of the $G$.  For each of these pieces,  for every edge $(i,j)$ in the piece, groups $\{i,\S_{ij}\}$ and $\{j,\S_{ij}\}$ are chosen where $\S_{ij}$ is a fixed $k-1$ sized set that does not contain $i$ or $j$. The idea is that, by obtaining the un-normalized happiness values $h_i = (s_i + \sum_{l \in \S_{ij}}s_l)^2+k^2\eta_i$ and $h_j = (s_j + \sum_{l \in \S_{ij}}s_l)^2+k^2\eta_j$ for these two groups over multiple rounds, one can compute a estimate of the difference of the corresponding scores $s_i - s_j$ with high confidence. As we only require the relative ordering to be correct, we can without loss of generality, set $\hat{s}_1 = 0$ and compute the remaining scores using the following procedure: For node $k$, we find the shortest path in $G$ that connects $1$ and $k$ and sum the differences along this path (w.h.p $G$ is connected and so there exists at least one path connecting $1$ and $k$). Each of these differences are estimates and hence the confidence in the sum of the these estimates depend on the diameter of $G$. We formally state the guarantee for \textsc{learnorder} below.

\begin{thm}
\label{thm:partitionLearn-score}(\textbf{PAC guarantee for} \textsc{learnorder})
Let $\W \in \bR^{n \times n}$ be a score based  compatibility matrix with score vector $\s \in \bR^n$. Let $\Delta_{\min} = \displaystyle \min_{i \neq j} |s_i - s_j|$, $\Delta = 2ks_{\min}\Delta_{\min} - \Delta_{\min}^2$ and let $\delta^* = \Big(1-\Big(\exp(-\frac{\delta}{k})\Big)\Big)/m$. Then, algorithm \textsc{LearnOrder} (Algorithm $1$) outputs a permutation $\hat{\sigma}$ whose ordering is same as that of $\s$ with probability at least $1 - \delta$ after $O\Big(\frac{|E|}{m}\frac{diam(G)^2}{\Delta^2}\ln(\frac{1}{\delta^*})\Big)$ rounds.
\end{thm}

\begin{figure*}[ht]
    \begin{subfigure}[b]{0.5\textwidth}
        \includegraphics[width=\textwidth]{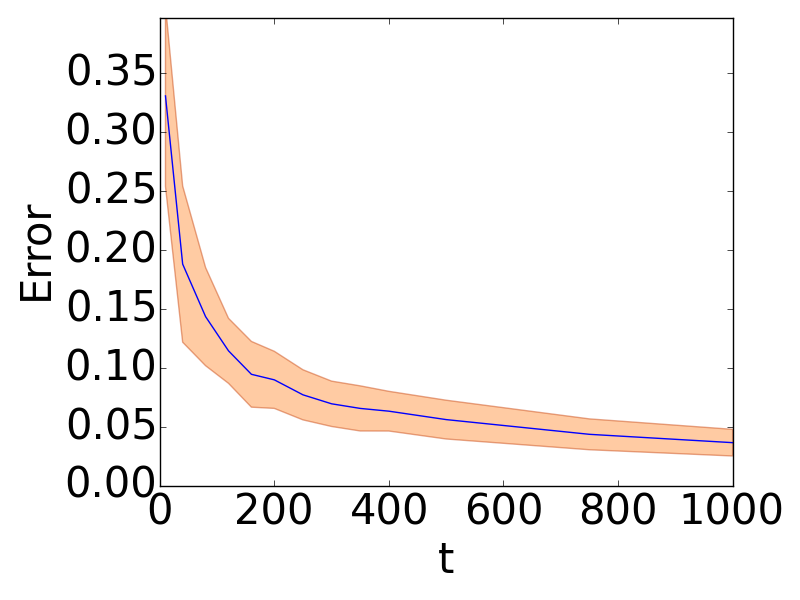}
    \end{subfigure}
    ~
    \begin{subfigure}[b]{0.5\textwidth}
        \includegraphics[width=\textwidth]{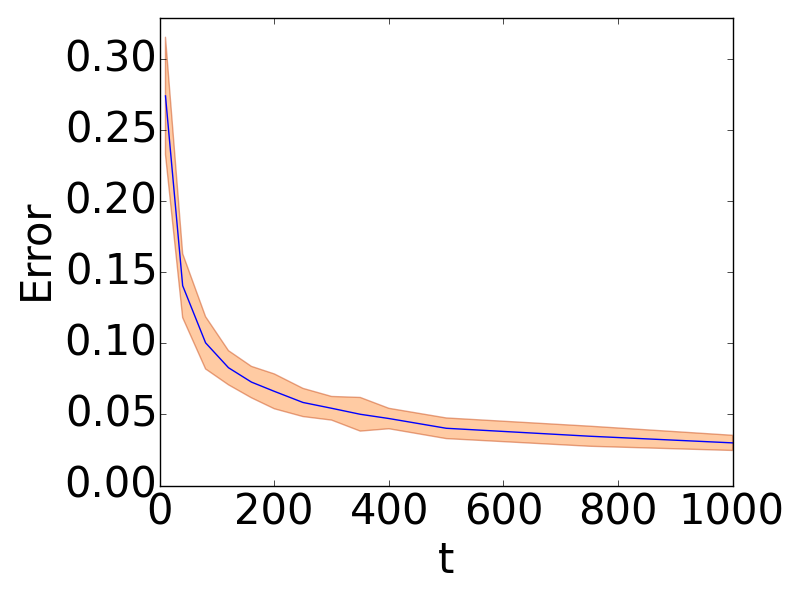}
    \end{subfigure}
    \caption{ Normalized error between the estimated and true weights (y-axis) of \textsc{LearnOrder}. The left plot is for a random graph  and the right is for a Facebook graph instance.}
    \label{fig:learn-weights}
\end{figure*}

\textbf{Proof Sketch:} For a chosen edge pair $(i,j)$, let $h_i$ and $h_j$ denote the unnormalized happiness values obtained by choosing the groups $(i,\S_{ij})$ and $(j,\S_{ij})$ where $|\S_{ij}| = k - 1$ and $i,j \notin \S_{ij}$.
We have, 
$$\sqrt{h_i} - \sqrt{h_j} = \sqrt{(s_i + a)^2 + k^2\eta_i} - \sqrt{(s_j + a)^2 + k^2\eta_j}$$
where $\eta_i$ and $\eta_j$ correspond to the bounded random noise and $a = \displaystyle \sum_{l \in \S_{ij}}s_l.$ 
If the noise $\eta_i, \eta_j$ were not present, then the difference $\sqrt{h_i} - \sqrt{h_j} = s_i - s_j$, which is what we want to estimate. Nevertheless, we can control the noise by playing this pair of groups for $O\Big(\frac{diam(G)^2}{\Delta^2}\ln(\frac{1}{\delta^*})\Big)$ rounds as in the Algorithm and averaging the happiness values obtained to obtain estimates $\hat{s}_{ij}$. In this case, after these many rounds, we have  with probability at least $1 - \delta^*,$
 $$ \frac{(s_i - s_j) - \Delta_{\min}}{diam(G)} \le \hat{s}_{ij} \le \frac{(s_i - s_j) + \Delta_{\min}}{diam(G)}\quad \forall i,j$$

When computing the estimate of a pair $(s_1,s_k)$ not in the edge set $E$, the algorithm sums up the estimated weights on the shortest path from $1$ to $k$. As the shortest path is at most $diam(G)$ long by definition, we obtain estimates for all pairs of the form $(s_1,s_k)$ such that all the estimated values $\hat{s}_{1k}$ satisfies with probability $1- \delta^{*}$,
$|\hat{s}_{1k} - s_{1k}| \le \Delta_{\min}~ \forall k.$ 

Thus under the above condition, if we fix $\hat{s}_1 = 0$ and obtain values for all other vertices, we can sort them to produce an ordering. It is easy to see that this ordering will exactly correspond to the ordering of the actual score vector $\s$. 

~\\
\textbf{Remark:} The sample complexity (i.e., the number of groups to be chosen) by the above Theorem depends on the diameter of the random graph $G$. It is known that for large enough $n$, $diam(G)$ is concentrated sharply around $2\frac{\log(n)}{\log(n/2)}$ and the number of edges behaves as $O(n\log(n))$.

\section{Experiments}
\label{sec:experiments}
We assess the quality of the \textsc{LearnOrder}  algorithm using simulated and real data.  The graph instances we chose were the following. (a) \textit{Random:} for this synthetically generated dataset, we fixed the group size $k = 4$ and the graph size $n = 16$. We added uniform noise between $[-1,1]$ to the feedback in each round and the score for each item was drawn uniformly at random from $\{1, ..., 10\}$. And (b) \textit{Social network:} a $16$ node instance was sampled from the Facebook friendship graph, built from an existing anonymized 
dataset of Facebook users' data (Leskovec and Krevl 2014). The dataset has $4039$ 
user nodes and $88234$ unweighted edges. We used the \textit{Jaccard 
similarity coefficient} of features such as education, hometown, 
language, etc to obtain scores for the 
users. The performance of \textsc{LearnOrder} is shown in Figure \ref{fig:learn-weights} (averaged over $30$ runs), and is in terms of the normalized error between the estimated score vector and the true score vector. It decreases as the number of rounds $t$ increases.

      In addition to showing that the learning algorithm indeed converges, the experiments add empirical support to the fact that the number of rounds needed to learn the true scores within 10\% normalized error is very practical (for instance ~70 rounds for the Facebook subgraph of size 16). Since the weight matrix was filled using a pairwise similarity measure (using demographical and other user specific metadata) and no intrinsic score was assumed, the experiment shows that the weight matrix is naturally low rank (allowing us to learn the scores very well) for this dataset. As a consequence, for this dataset, we could infer that the users' intrinsic characteristics reasonably determine who they are friends with.


\section{Conclusions}
We studied the problem of grouping a set of users using their pairwise compatibilities. We first showed hardness and inapproximability results when no assumptions on the pairwise compatibility values are made. We then studied the intrinsic score model for the compatibility, a model that is not only simple and straight forward but also very similar to the popular Bradley--Terry--Luce (BTL) model for pairwise comparisons. Under this model, we related the optimal groupings to  homophilous and heterophilous groupings which are well studied in the psychology literature. 
We proposed the \textsc{LearnOrder} algorithm, which after choosing a small number of groups, adaptively learns the best ordering corresponding to the score vector of the pairwise compatibility matrix. Our experiments on both synthetic and real datasets demonstrate the efficacy of our algorithm.

We note that there may be several applications where the pairwise compatibilities between users/items may depend on multiple features (instead of one) and the pairwise compatibility matrix can in general be low rank (instead of being score based). In such cases, our framework can be slightly modified to incorporate a matrix completion subroutine to recover the low rank compatibility matrix. However, the results regarding the optimality of the homophilous/heterophilous partitions do not follow. The analysis of this is beyond the scope of the current work. 

In the future, we would like to consider other relevant structures for the happiness index and develop algorithms for the same, possibly with statistical as well as computational guarantees. 
\balance
\bibliographystyle{abbrv}
\bibliography{sigproc}  
%
\end{document}